\documentclass{article}

\usepackage[preprint]{neurips_2023}

\usepackage[utf8]{inputenc} % allow utf-8 input
\usepackage[T1]{fontenc}    % use 8-bit T1 fonts
\usepackage{hyperref}       % hyperlinks
\usepackage{url}            % simple URL typesetting
\usepackage{booktabs}       % professional-quality tables
\usepackage{amsfonts}       % blackboard math symbols
\usepackage{nicefrac}       % compact symbols for 1/2, etc.
\usepackage{microtype}      % microtypography
\usepackage{xcolor}         % colors

\usepackage{amsmath}
\usepackage{amssymb}
\usepackage{mathtools}
\usepackage{amsthm}

\usepackage{makecell}
\usepackage[symbol]{footmisc}

\usepackage{algorithm}
\usepackage{algorithmic}

\DeclareMathOperator{\erfc}{erfc}

\usepackage[
raggedright
]{sidecap}   
\sidecaptionvpos{figure}{t}

\theoremstyle{plain}
\newtheorem{theorem}{Theorem}[section]
\newtheorem{proposition}[theorem]{Proposition}
\newtheorem{lemma}[theorem]{Lemma}
\newtheorem{corollary}[theorem]{Corollary}
\theoremstyle{definition}

\theoremstyle{plain}
\newtheorem{remark}[theorem]{Remark}

\title{The Disharmony between BN and ReLU Causes Gradient Explosion, but is Offset by the Correlation between Activations}

\author{%
  Inyoung Paik, Jaesik Choi  \\
  KAIST, South Korea\\
  \texttt{\{humandream, jaesik.choi\}@kaist.ac.kr} \\
}

\begin{document}

\maketitle

\begin{abstract}
Deep neural networks, which employ batch normalization and ReLU-like activation functions, suffer from instability in the early stages of training due to the high gradient induced by temporal gradient explosion. In this study, we analyze the occurrence and mitigation of gradient explosion both theoretically and empirically, and discover that the correlation between activations plays a key role in preventing the gradient explosion from persisting throughout the training. Finally, based on our observations, we propose an improved adaptive learning rate algorithm to effectively control the training instability. 
\end{abstract}

\section{Introduction}
The success of deep neural networks stems from their remarkable expressive power, which grows exponentially with increasing depth \citep{depth_width_tradeoff}. However, the deep hierarchical architecture of these networks can give rise to the problem of exploding/vanishing gradients \citep{goodfellow2016}, which significantly impairs performance and may render training infeasible.

Some studies \citep{exploding_exist,mean_field} have shown that gradient explosion can occur in modern neural networks, showing that a stable flow of the forwarding signal does not guarantee a stable flow of backward gradient. However, since the advent of normalization methods \citep{BatchNorm,InstanceNorm,LayerNorm,GroupNorm,weight_standardization}, it appears that gradient explosion is no longer a serious issue and is often regarded as a problem of the past.

However, we note that gradient explosion actually exists in the variety of modern deep neural networks with ReLU and batch normalization \citep{BatchNorm} in the initialization state. However, due to the rapid changes in training dynamics, different terminologies have been used to describe this problem, like `large gradient magnitude at the early stage of training' or `instability of deep neural networks in the early phase of training' \citep{early_phase,warmup}. It is also often addressed as an `instability problem of large batch training' \citep{warmup}, as the severity of the issue escalates when using a large batch size (alongside a correspondingly high learning rate) during training.

The disparity between theory and practice has hindered the resolution of this problem, and practical solutions have been developed relying on empirical experience and intuition. In this paper, we aim to offer a more precise diagnosis of the problem and propose a more efficient solution derived from this diagnosis.

\begin{figure}
\includegraphics[width=0.5\textwidth]{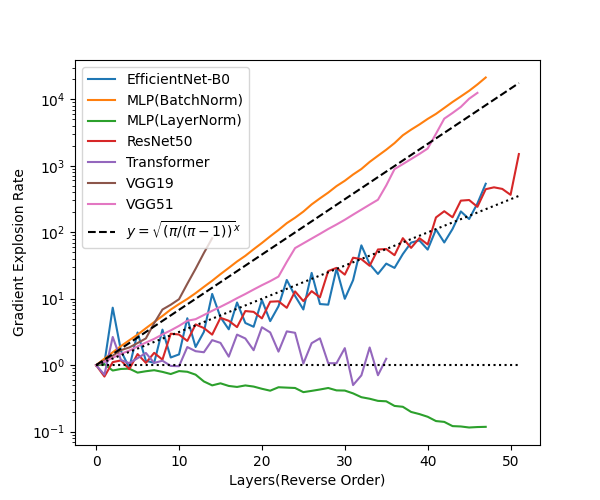} % \textbf{}
\includegraphics[width=0.5\textwidth]{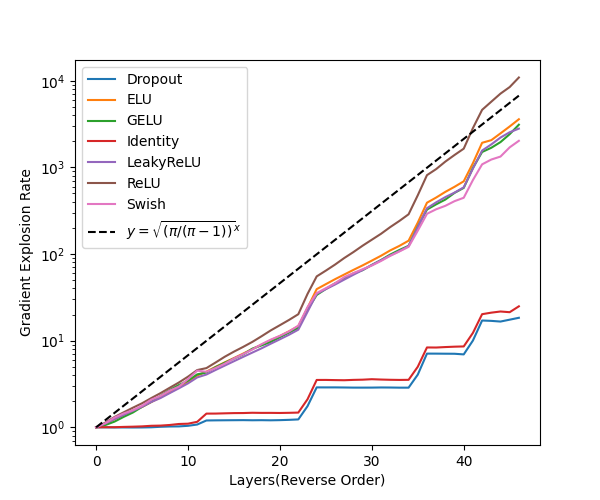} % [width=\textwidth]
\caption{Gradient explosion rate ($\sqrt{\mathrm{Var}(g^n)/\mathrm{Var}(g^N)}$) of deep neural network models at the initialization state corresponding to different architectures (left) and activation functions (right). The explosion rate is approximately $\sqrt{\pi/(\pi-1)}$ in vanilla networks with batch normalization. However, it is lower in architectures with residual connections \citep{ResNet}, which reduces the effective depth. \citep{exploding_exist} No evidence of gradient explosion was observed in architectures without batch normalization, including transformer-based architectures \citep{vit}. \\
The figure on the right is plotted using VGG \citep{VGG} architecture with different activation functions. Smoother variants of ReLU \citep{leakyrelu,elu,swish,gelu} exhibit lower explosion rates due to their flatter behavior near the origin. It is worth noting that gradient explosion does not occur with DropOut \citep{dropout}, which can be regarded as a ReLU that randomly blocks signals.}
\label{fig:exploding}
\end{figure}

\section{How Gradient Explosion Occurs}
As weights are repeatedly multiplied during forward/backward propagation, the exploding or vanishing gradient problem is commonly believed to be caused by excessively large or small parameters \citep{exploding_b,exploding_RNN}. Consequently, it has often been regarded as a problem related to the initialization and maintenance of appropriate weight scales. \citep{he_initialization} assumed that initializing weight with $N(0,\sqrt{2/n_{out}})$ would maintain similar variances in both forward and backward propagation. In this perspective, this problem was seen as potentially `solved' with the introduction of (batch) normalization \citep{BatchNorm}, which automatically corrects suboptimal choices of weight scales.

However, it is important to consider how the activation function affects the input distribution. Figure \ref{fig:relu} illustrates that using the positive part of the activation differs from randomly blocking selected activations. In short, \citep{he_initialization} described the scenario assuming DropOut \citep{dropout} as the activation function rather than ReLU. In this case, both the forward and backward signals are roughly halved, preventing the occurrence of either exploding or vanishing gradients. However, real activation functions are highly dependent on the forward signal while being almost uncorrelated with the backward gradient. This discrepancy can lead to differential effects on the variance during forward and backward propagation \citep{exploding_exist}.

\begin{figure}[ht]
\centering
\includegraphics[width=0.9\textwidth]{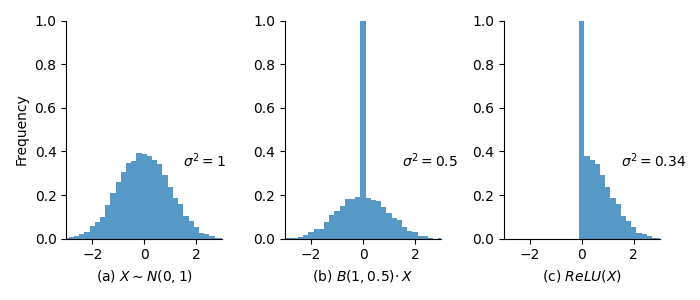}
\caption{(a) Sample of normally distributed input. (b) Hypothetical output distribution assuming that ReLU randomly drops half of the input. (c) Real distribution after applying ReLU, which exhibits a smaller variance than that in (b). \citep{he_initialization} assumed the post-ReLU distribution to be similar to that depicted in (b), in which half the variance is compared to the input. It leads to the conclusion that selecting optimal weights to maintain input variance also ensures the stability of the gradient variance since ReLU also halves the gradient during backpropagation. However, the correlation between the input and the activation function can affect the input variance and gradient variance differently. With batch normalization, which restores the variance reduced by the activation function, the backpropagating gradient may exponentially increase while the forward propagation remains stable.}

\label{fig:relu}
\end{figure}

\begin{proposition}
\label{thm:prop1}
Let $X = \mathrm{ReLU}(Y) = max(Y, 0)$, where $Y \sim N(\mu,\sigma^2)$. Then, we have:

\begin{equation}
\begin{aligned}
\mathbb{E}(X) = \mu(1-\phi (-\frac{\mu}{\sigma})) + \frac{\sigma}{\sqrt{2\pi}}e^{-\frac{\mu^2}{2\sigma^2}} \\
\end{aligned}
\end{equation}

\begin{equation}
\begin{aligned}
\mathrm{Var}(X)
    & = (\sigma^2+\mu^2)(1-\phi(-\frac{\mu}{\sigma})) + \mu \frac{\sigma}{\sqrt{2\pi}}e^{-\frac{\mu^2}{2\sigma^2}}  - (\mu(1-\phi (-\frac{\mu}{\sigma})) + \frac{\sigma}{\sqrt{2\pi}}e^{-\frac{\mu^2}{2\sigma^2}})^2
\end{aligned}
\end{equation}

where $\phi(x) = \int_{-\infty}^{x}\frac{1}{\sqrt{2\pi}}e^{-x^2/2}$ is the cumulative distribution function of the standard normal distribution.
\end{proposition}

\begin{proof}
See Appendix \ref{prf:prop1} 
\end{proof}

\begin{theorem}
\label{thm:basic}
Consider the repetitive neural network architecture $f^n : \mathbb{R}^{d_n} \longmapsto \mathbb{R}^{d_{n+1}}$. Let $x^{n+1} = f^n(x^n) = \mathrm{BatchNorm}(W^n(\mathrm{ReLU}(x^n))+b^n)=\frac{W^n(\mathrm{ReLU}(x^n))+b^n-\hat{\mu}}{\hat{\sigma}}\gamma^{n+1}+\beta^{n+1}$, where $\gamma^{n+1},\beta^{n+1} \in \mathbb{R}^{d_{n+1}}$ are affine transformation parameters, and $\hat{\mu},\hat{\sigma} \in \mathbb{R}^{d_{n+1}}$ are the estimated mean and standard deviation. Let input values $x^n_i \sim N(\mu^n, (\sigma^n)^2)$ be independent, the gradient of the output ($g^{n+1}$) is uncorrelated to the input, and all variables ($x,w,g$) have finite variance. Ignoring the sampling error of batch normalization, we have:

\begin{equation}
C(R) \leq \mathbb{E}_w\bigg[\frac{\sum_i \mathrm{Var}_x(x_i^n) \mathrm{Var}_g(g^n_i)}{\sum_j \mathrm{Var}_x(x_j^{n+1}) \mathrm{Var}_g(g^{n+1}_j)}\bigg] 
\end{equation}
where

\begin{equation}
C(R) = \frac{\erfc(-R)} { (1+2R^2)\erfc(-R) + 2R \frac{e^{-R^2}}{\sqrt{\pi}}  - (R\erfc(-R) +   \frac{e^{-R^2}}{\sqrt{\pi}} )^2}
\end{equation}

and $R := \frac{\mu^n}{\sqrt{2}\sigma^n}$, $\erfc (x) = \frac{2}{\pi} \int_{x}^{\infty}\frac{1}{\sqrt{2\pi}}e^{-x^2/2}$ is the complementary error function. 

\end{theorem}

\begin{proof}
See Appendix \ref{prf:basic}.
\end{proof}

\begin{corollary}
\label{thm:R}
$C(R)$ strictly decreased with $\lim_{R \to -\infty} C(R) = \infty $ (blocked), $C(0) = \pi/(\pi-1)$ (zero-centered), and $\lim_{R \to \infty} C(R) = 1 $ (pseudo-linear).  
\end{corollary}

\begin{proof}
The cases $R=0$ and $R\to \infty$ is trivial. See Appendix \ref{prf:R} for the case of $R\to -\infty$, 
\end{proof}

\begin{theorem}
\label{thm:upper_bound}
Under the assumptions of Theorem 1, let the weight also be normally distributed. Then, for any $0<\delta<1$, with probability at least $1-d_{n+1} \mathrm{exp}(-(d_n/2)(\delta -1-\mathrm{ln}(\delta)))$, we have:  

\begin{equation}
C(R) \leq \mathbb{E}_w\bigg[\frac{\sum_i \mathrm{Var}_x(x_i^n) \mathrm{Var}_g(g^n_i)}{\sum_j \mathrm{Var}_x(x_j^{n+1}) \mathrm{Var}_g(g^{n+1}_j)}\bigg] \leq C(R) 
\bigg(1 + \frac{2(1+\frac{\mu_w^2}{\sigma_w^2}) (1+2\frac{\mu_w^2}{\sigma_w^2})}{d_n^3 \delta^3} \bigg)
\end{equation}

where $\mu_w,\sigma_w^2$ are the mean and variance of the weight distribution.

\end{theorem}

\begin{proof}
See Appendix \ref{prf:upper_bound}.
\end{proof}

\begin{remark}
    The upper bound presented in Theorem \ref{thm:upper_bound} is stochastic due to the possibility of batch normalization statistics diverging. However, the probability of divergence exponentially decreases with respect to the network width ($d_n$). For example, when setting $d_n=d_{n+1}=10^2$, $\mu_w=0$, and $\delta=0.3$, there is at least $1-10^{-9}$ chance of having a set of parameters that exhibits the explosion rate between $C(R)$ and $C(R)(1+10^{-4})$. Our results align with those of \citep{mean_field}, which calculated the explosion rate to be $\sqrt{\pi/(\pi-1)}$ under the assumption of infinite width, and zero-centered i.i.d. normal activations and parameters.
\end{remark}

\begin{remark}
    The network exhibits a `pseudo-linear' behavior when the majority of activations lie within the linear section of the nonlinearity and the nonlinearity can be effectively approximated by a linear function \citep{exploding_exist}. As $R \to \infty$, ReLU approaches the identity function, resulting in a layer that does not cause gradient explosion. However, it also fails to act as a valid layer contributing to the effective depth.
\end{remark}

\begin{remark}
    Independence between inputs is a widely adopted assumption in deep learning theory \citep{indep_ex2, indep_ex3, exploding_exist, indep_ex1, mean_field} due to the inherent challenges of defining and analyzing the dependency between input distributions. However, both practically and theoretically, we observed that the correlation between activation can have a significant impact on the gradient flow, even leading to a weak vanishing gradient. Theorem \ref{thm:corr} demonstrates this phenomenon by examining the extreme opposite case, and the negative slope of $\mathrm{Var}(g)$ is illustrated in Figure \ref{fig:xwg}.
\end{remark}

\begin{theorem}
\label{thm:corr}
Under the assumptions of Theorem 1, let inputs are perfectly correlated. For zero-centered case, we have:

\begin{equation}
C_{full}(0) = \frac{\pi}{1+\pi} \leq \mathbb{E}_w\bigg[\frac{\sum_i \mathrm{Var}_x(x_i^n) \mathrm{Var}_g(g^n_i)}{\sum_j \mathrm{Var}_x(x_j^{n+1}) \mathrm{Var}_g(g^{n+1}_j)}\bigg]
\end{equation}

where equality holds at $d_n \to \infty$. 
\end{theorem}

\begin{proof}
See Appendix \ref{prf:corr}
\end{proof}

\begin{figure}
\includegraphics[width=0.9\textwidth]{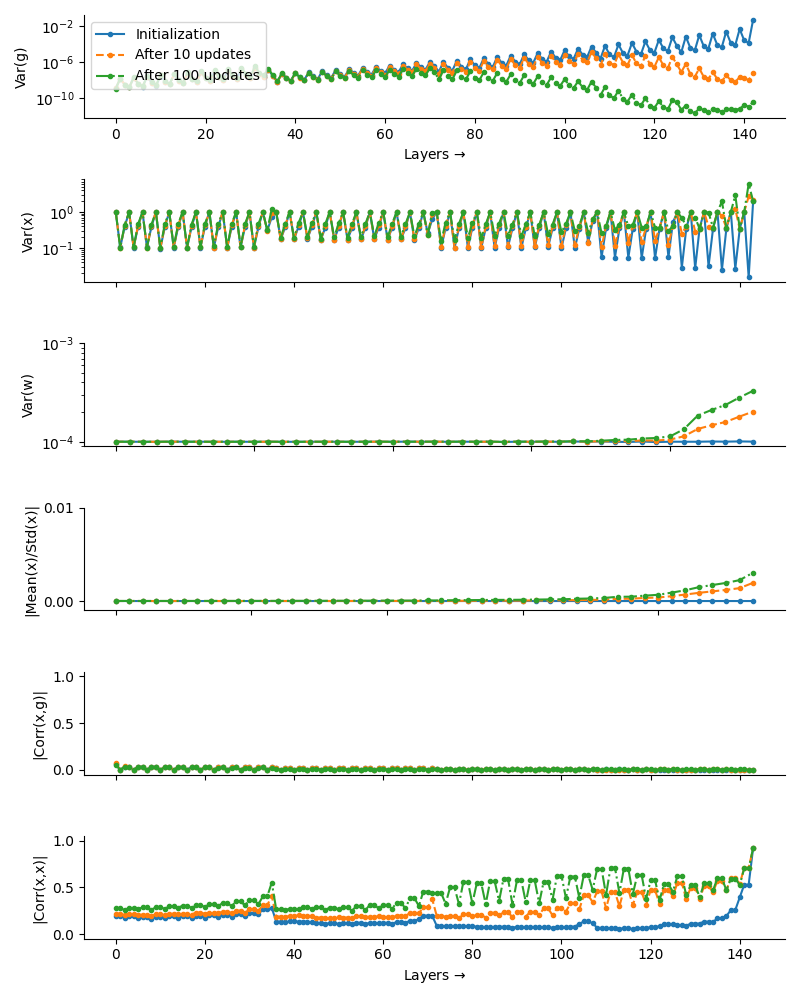}
\caption{A sample chart of the variance of the gradient, activation, weight, mean-std ratio, the correlation between activation and gradient ($\sum_{i=1}^C |\mathrm{Corr}(x_i, g_i)|/C$), and the average correlation between activations ($\sum_{i\neq j=1}^C |\mathrm{Corr}(x_i, x_j)|/(C(C-1))$) during the early stages of training. Note that the left side represents the layer closer to the output. All weights are initialized from $N(0,0.01^2)$ for better visualization. \\
In brief, the top 2 graphs represent the evolution of the gradient and activation flow, while the remaining graphs represent internal values that potentially contribute to gradient explosion or vanishing. Various factors may help mitigate gradient explosion during the initialization stage. But the correlation between activations appears to be the primary contributor to gradient decrease. It modifies the variance after linear transformation and reduces the norm of the gradient flow while maintaining stability in the forward flow, given that the variance of the sum of dependent variables differs from that of independent variables. (See Theorems \ref{thm:basic} and \ref{thm:corr}).}
\label{fig:xwg}
\end{figure}

\begin{SCfigure}
\includegraphics[width=6cm]{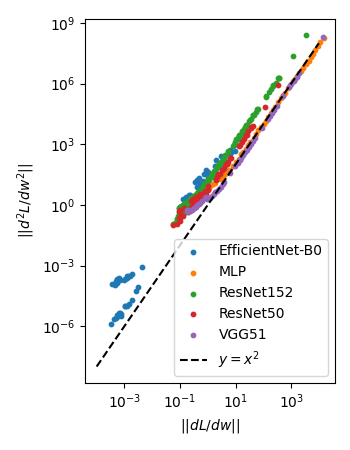}
\caption{The scatter plot that shows a relationship between the norm of the gradient ($||\frac{dL}{dw}||$) and the second-order derivative ($||\frac{d^2L}{dw^2}||$) for various architectures. Both axes are plotted on a log scale, and the dashed line represents $y=x^2$. To alleviate computational overhead, we randomly sampled 1,000 parameters for each layer to compute the first and second-order derivatives. We found that the magnitude of the second-order derivative is approximately proportional to the square of the gradient, in line with the expectation from Lemma \ref{thm:hessian}.}
\label{fig:hessian}
\end{SCfigure}

A neural network employing batch normalization and ReLU may experience a significant gradient explosion at initialization, which, however, rapidly diminishes during training. Maintaining a stable gradient flow is not a `natural' occurrence as several factors may contribute to gradient explosion or vanishing, and these can change over the course of training. These include:

\begin{itemize}
    \item The probability of passing through the ReLU activation function, or the degree of `pseudo-linearity' discussed in Corollary \ref{thm:R}.
    \item The degree of correlation between activations, as explored in Theorem \ref{thm:corr}.
    \item Changes in the network width ($d_n$). Note that our theorem describes the \textbf{sum} across the width instead of the average. For instance, Figure \ref{fig:exploding} demonstrates a step-wise change in the gradient norm at the downsampling layers of the VGG network.
    \item Changes in the magnitude of weight parameters, considering that $dL/dw^n_{ij}= w_{ij}^n dL/dx^{n+1}_j$.
    \item The existence of other unforeseen dependencies among internal variables, such as activations, weights, and gradients.
\end{itemize}

However, the correlation between input activations, as observed in Figure \ref{fig:xwg}, appears to be the most critical factor that reduces the gradient. Many of the assumptions posited in Theorem \ref{thm:basic} seem to remain valid even after the neural network has been trained. Internal variables such as weights, activations, and gradients tend to remain zero-centered and exhibit minimal correlation. Yet, a neural network tends to produce numerous redundant and duplicate signals during training, necessitating `neural network pruning' \citep{pruning}.

\subsection{Why Vanilla Gradient Descent Fails in the Face of Gradient Explosion}

Gradient explosion is not a numerical or computational error. As shown in \citep{exploding_exist}, a small change in the lowest layers can make a huge difference in the upper layers. Yet, we cannot always rely on this value because the gradient is not always a trustworthy estimator for the optimal step size, as it also correlates heavily with the curvature of the loss landscape. \citep{loss_curvature} Take the basic second-order optimization algorithm, $\Delta w = (H+\epsilon I)^{-1}J$ \citep{newton}, where $H$ is the Hessian matrix, $J$ is the Jacobian, and $\epsilon$ is a stability hyperparameter. Since computing the Hessian matrix for a deep neural network with millions or billions of parameters is almost impossible, we assume it to be reasonably bounded by a single scalar value, commonly known as the `learning rate'.

This assumption is approximately valid for parameters in parallel, e.g. the set of neural network parameters within the same layer. However, for parameters connected in series, the second-order derivative can increase \textbf{faster} than the first-order derivative, leading to a smaller optimal step size for parameters with larger gradients. For instance, consider $y=w_2w_1x$, where $w_1=100$ and $w_2=1$. Although $w_2$ has a gradient that's $100$ times larger than that of $w_1$, it has an optimal step size that's $100$ times smaller when considering the second-order term. More formally:

\begin{lemma}
    
\label{thm:hessian}
Let $f$ be a piecewise linear neural network with a single output node ($d_N=1$), and let the loss function be nonlinear and twice differentiable. If the gradient is normally distributed, for $1 \leq n,m \leq N$,

\begin{equation}
\frac{\mathrm{Var}(h_i^n)}{\mathrm{Var}(h_j^m)} = \Big( \frac{\mathrm{Var}(g_i^n)}{\mathrm{Var}(g_j^m)} \Big)^2 
\end{equation}

where $h^n:= d^2L/(dx^n)^2$ is a second-order derivative.

\label{hessian}
\end{lemma}

\begin{proof}
See Appendix \ref{prf:hessian}
\end{proof}

Since we observed that the issue of exploding gradients in modern deep learning architectures is not caused by exponentially varying weight or activation, we can anticipate that the second-order derivative experiences an explosive growth rate proportional to the square of the gradient explosion rate. Figure \ref{fig:hessian} presents empirical observations from various architectures, and we hypothesize that this approximately holds true, at least when describing the gradient explosion.

\section{Previous Solutions}
Although the diagnosis of this issue may not have been exhaustive, fortunately, this problem is temporary. While there may exist a mild gradient explosion or vanishing in the middle of the training, practical observations suggest that it is not necessary to normalize it. \citep{reality_check,loss_curvature} Therefore the WarmUp method \citep{warmup} can serve as an effective solution by assigning a very small learning rate during the early stages of training and \textbf{`waiting'} until the learning process stabilizes. 

The WarmUp technique is a strong baseline that widely employed in both small and large batch training. We hypothesize that it could overcome any degree of instability given a sufficiently long WarmUp schedule, as shown in \citep{reality_check,warmup_imagenet}. However, instead of keep extending the schedule, more direct solutions like LARS \citep{LARS} can be preferred in severe cases, as these methods typically offer better stability and require only a negligible amount of additional computation.

Layer-wise Adaptive Rate Scaling (LARS, \citep{LARS}) is a technique designed to maintain the ratio between the gradient and weight of each layer. It has been successfully implemented in several large-scale experiments \citep{lars_ex1,lars_ex2,lars_ex3,lars_ex4,lars_ex5}. However, its application in small-batch experiments has been limited due to the constraints it imposes on training and its potential to induce minor performance degradation \citep{reality_check}. Additionally, it has been observed that LARS needs to operate in conjunction with WarmUp, suggesting that training instability is not sufficiently managed by LARS alone \citep{CLARS}.

Consequently, various enhancements to LARS have been proposed. \citep{LAMB} suggested LAMB, which modified the formula to clip extreme values and applied the technique to the Adam optimizer \citep{adam} for large-scale language model. \citep{LAMBC} introduced LAMBC, which clips the trust ratio of LAMB \citep{LAMB}. \citep{CLARS} developed Complete Layer-wise Adaptive Rate Scaling (CLARS), which utilizes the average per-sample gradient of the norm rather than the minibatch gradient. \citep{AGC} suggested Adaptive Gradient Clipping (AGC), applying LARS only when the gradient norm exceeds a certain threshold and employing a unit-wise gradient norm instead of the layer-wise one. These approaches are compared in Table \ref{tab:compare}. It should be noted that while these techniques differ in the specifics, they are all rooted in the concept of \textbf{`maintaining the ratio between weight and gradient'}, i.e., $\Delta w \propto \frac{||w||}{||g||+\epsilon}g$.

The gradient explosion occurs only when \textbf{both} batch normalization and activation functions are utilized. Although there is an inherent limitation in resolving this problem by linearizing the activation function, \citep{exploding_exist} it is noteworthy that removing the batch normalization layer can be a valid solution if layer statistics can be normalized and maintained by alternative methods. \citep{LayerNorm,GroupNorm,vit,fixup,gradinit}

\section{Solution}
Based on our observations, we strive to incorporate the following principles to address the issue of gradient explosion:

\begin{itemize}
    \item We require a step size that is proportional to the gradient norm for parameters in parallel, while simultaneously being inversely proportional to the gradient norm for parameters in series. As such, we need to employ a layer-wise optimization method to implement this inverse relationship.
    \item For parameters connected in series, the square of the gradient norm seems to better represent the second-order derivative (curvature) than the gradient norm itself.
    
    \item Since the problem is temporal, a temporal solution is more suitable as it is inevitable to encounter a certain degree of error in calculating the optimal step size between layers. We found that the clipping method is generally more effective than the scaling method, even without the assistance of WarmUp.
\end{itemize}

Our proposed algorithm directly embodies these insights. Specifically, we drew from the formula for basic second-order optimization, $\Delta w = (H+\epsilon I)^{-1}J$, where the Hessian is approximated by the square of the gradient norm. To prevent the scale of the gradient from being affected by the weight norm, we normalized it by the norm of weight. We then established an upper limit for the learning rate using this value. The procedure is detailed in Algorithm \ref{alg:lalc}, and the PyTorch \citep{pytorch} implementation is provided in Appendix \ref{appendix:detail}.

\begin{figure}[h]
\begin{algorithm}[H]
\caption{LALC}
\label{alg:lalc}
\begin{algorithmic}

\REQUIRE $w^l$ : Weight matrix of layer l 
\REQUIRE $\gamma_t$ : Learning rate at step t
\REQUIRE $f$ : Gradient optimization algorithm (SGD, Adam, etc.) 
\REQUIRE $\eta$, $\epsilon$  : Hyperparameters
\WHILE{$t<T$ for each layer l}
\STATE $g_t^l \leftarrow \frac{dL}{dw_t^l}$
\STATE $m_t^l \leftarrow f(g^l, w^l_t)$
\STATE $\lambda^l_t \leftarrow  \frac{1}{\eta ||m_t^l||^2/||w_t^l||^2 + \epsilon}$
\STATE $w^l_{t+1} \leftarrow w^l_t - min(\gamma_t,\lambda^l_t) m_t^l$
\ENDWHILE
\end{algorithmic}
\end{algorithm}
\caption{
Our proposed algorithm, inspired by the basic second-order optimization formula. See Section 4 for a detailed explanation, and Appendix \ref{appendix:detail} for the PyTorch \citep{pytorch} implementation.}

\end{figure}

\begin{table}[h]
\centering
\begin{tabular}{| c | c |}
\hline
Algorithm & $\Delta W$  \\ \hline
\makecell{Gradient Descent \\ (Baseline)} & $\gamma_t m_t^l$   \\ \hline
\makecell{LARS \\ \citep{LARS}}  & $\gamma_t \frac{\eta ||w_t^l||}{||g_t^l||+\beta ||w_t^l||+\epsilon}m_t^l$         \\ \hline
\makecell{LARS/LAMB \\ \citep{LAMB}}   & $ \gamma_t \frac{\phi(||w_t^l||)}{||m_t^l||+\epsilon}m_t^l$   \\ \hline
\makecell{LAMBC \\ \citep{LAMBC}}  &  $ \gamma_t min \big( \frac{\phi(||w_t^l||)}{||m_t^l||+\epsilon}, \mu \big) m_t^l$           \\ \hline
\makecell{CLARS \\ \citep{CLARS}}  & $\gamma_t \frac{\eta ||w_t^l||}{\sum_b ||m_{b,t}^l|| / B +\epsilon}m_t^l$            \\ \hline
\makecell{AGC \\ \citep{AGC}}   & $ \gamma_t min \big(\frac{\eta ||w_t^{unit}||}{||m_t^{unit}||+\epsilon}, 1 \big) m_t^{unit}$           \\ \hline
\makecell{LALC \\ (ours)}   & $min \Big( \gamma_t , \frac{1}{\eta ||m_t^l||^2/||w_t^l||^2 + \epsilon}  \Big) m_t^l $            \\ \hline
\end{tabular} \\
\label{tab:compare}
\caption{
A simple comparison between adaptive rate algorithms, where $\beta$ represents the weight decay rate, $\phi$ is a `scaling function' as referred to in \citep{LAMB},  $\mu$ is a clipping hyperparameter, and the rest of the notation follows algorithm \ref{alg:lalc}. Note that CLARS utilizes an average `per-sample gradient' value instead of a minibatch gradient \citep{CLARS}, and AGC employs a `unit-wise ratio' instead of considering all parameters in each layer. \citep{AGC} For precise definitions and formulations, please refer to the original papers.}
\end{table}

\section{Experiments}
In all of our experiments, we utilized the ResNet50 model \citep{ResNet}. For LARS \citep{LARS}, CLARS \citep{CLARS}, and AGC \citep{AGC}, we used a smaller $\eta$ for larger batch sizes. We conducted a hyperparameter search in order of 10, thus careful hyperparameter tuning like \citep{reality_check} could potentially enhance performance. LARC and LAMBC was more robust to changes in batch size, possibly due to their clipping-based methodologies. The experimental results are presented in Figure \ref{fig:CIFAR} and \ref{fig:ImageNet}. Please refer to Appendix \ref{appendix:detail} for details. 

\begin{figure}
         \includegraphics[width=0.49\textwidth]{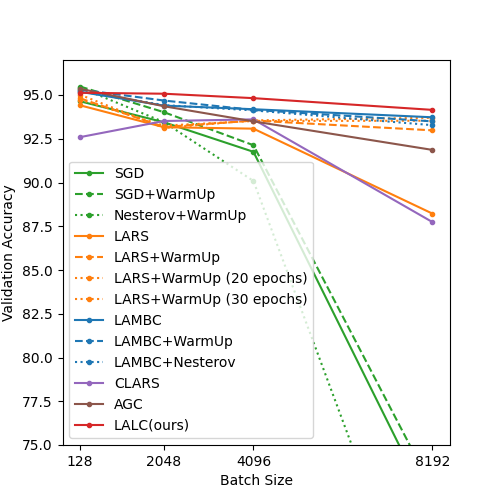}
         \includegraphics[width=0.49\textwidth]{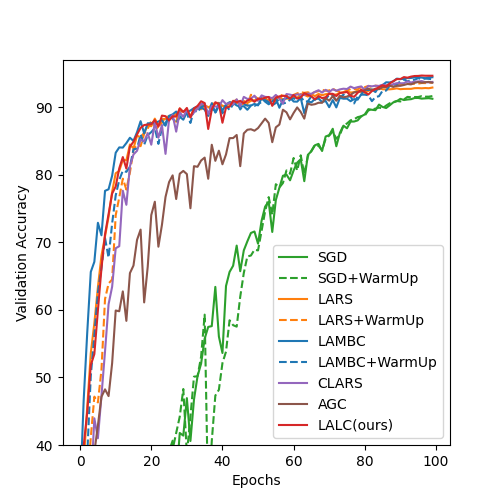}
     \caption{Results of the CIFAR10 experiment for (a) different batch sizes and (b) sample training curves for 4k batch size. All experiments were averaged over 3 runs. Please refer to Appendix \ref{appendix:detail} for details. We observed performance degradation irrespective of the algorithm at $\geq$8k batch size since CIFAR10 only contains 50k training data points.}

\label{fig:CIFAR}
\end{figure}

\begin{figure}
    
\centering
\includegraphics[width=0.7\textwidth]{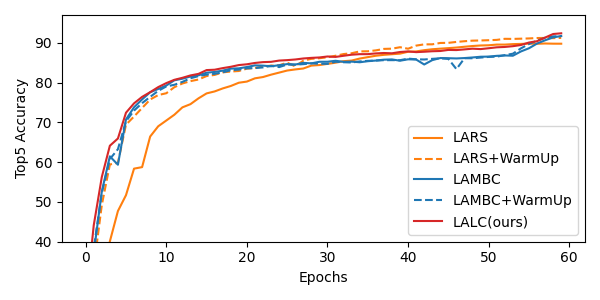}
 \caption{the results of an ImageNet experiment conducted with a batch size of 8k. We observed the similar tendency and training curves with CIFAR10 experiment.}

\label{fig:ImageNet}
\end{figure}

\section{Discussion}

In this paper, we have investigated the occurrence of gradient explosion in neural network architectures utilizing ReLU activation and batch normalization. Our analysis encompasses a wider range of input and weight distributions; however, our findings may not be fully generalized to practical neural networks. Notably, we have observed that the convolutional layer tends to exhibit a lower explosion rate compared to fully connected layers, as demonstrated in Figure \ref{fig:exploding}. This discrepancy may be attributed to the correlation between activations in convolutional neural networks, where adjacent pixels often exhibit similar activation due to spatial considerations.

\bibliography{neurips_2023}
\bibliographystyle{icml2023}

%%%%%%%%%%%%%%%%%%%%%%%%%%%%%%%%%%%%%%%%%%%%%%%%%%%%%%%%%%%%%%%%%%%%%%%%%%%%%%%
%%%%%%%%%%%%%%%%%%%%%%%%%%%%%%%%%%%%%%%%%%%%%%%%%%%%%%%%%%%%%%%%%%%%%%%%%%%%%%%
% APPENDIX
%%%%%%%%%%%%%%%%%%%%%%%%%%%%%%%%%%%%%%%%%%%%%%%%%%%%%%%%%%%%%%%%%%%%%%%%%%%%%%%
%%%%%%%%%%%%%%%%%%%%%%%%%%%%%%%%%%%%%%%%%%%%%%%%%%%%%%%%%%%%%%%%%%%%%%%%%%%%%%%
\newpage
\appendix
\onecolumn

\appendix
\section{Appendix}
\subsection{Proof of Proposition \ref{thm:prop1}}
\label{prf:prop1}
We use the moment generating function and its properties. 

\begin{equation}
\begin{aligned}
M_X(t) :&= \mathbb{E}[e^{tX}] \\
    & = e^0P(X=0)+\int^\infty_0 e^{tx} \frac{1}{\sqrt{2\pi}\sigma} e^{-\frac{(x-\mu)^2}{2\sigma^2}}dx \\
    & = P(N(\mu, \sigma^2) \leq 0) + \int^\infty_0 \frac{1}{\sqrt{2\pi}\sigma} e^{-\frac{(x-\mu-\sigma^2t)^2-2\mu\sigma^2t-\sigma^4t^2}{2\sigma^2}}dx\\
    & = \phi(-\frac{\mu}{\sigma}) + e^{\mu t+\frac{\sigma^2t^2}{2}}\int^\infty_0 \frac{1}{\sqrt{2\pi}\sigma} e^{-\frac{(x-\mu-\sigma^2t)^2}{2\sigma^2}}dx\\
\end{aligned}
\end{equation}

where $\phi(x) = \int_{-\infty}^{x}\frac{1}{\sqrt{2\pi}}e^{-x^2/2}$ is the cumulative distribution function of the standard normal distribution. Substituting $x' = (x-\mu-\sigma^2t)/\sigma$, 

\begin{equation}
\begin{aligned}
    & = \phi(-\frac{\mu}{\sigma}) +  e^{\mu t+\frac{\sigma^2t^2}{2}}\int^\infty_{-\frac{\mu+\sigma^2t}{\sigma}} \frac{1}{\sqrt{2\pi}}e^{-\frac{x'^2}{2}}dx'\\
    & = \phi(-\frac{\mu}{\sigma}) +  e^{\mu t+\frac{\sigma^2t^2}{2}}(1-\phi(-\frac{\mu+\sigma^2t}{\sigma}))
\end{aligned}
\end{equation}

We calculate the first and second moment from the moment generating function. 

\begin{equation}
\begin{aligned}
\frac{d}{dt}M_X(t)
    & = (\mu+\sigma^2 t) e^{\mu t+\frac{\sigma^2t^2}{2}} (1-\phi(-\frac{\mu+\sigma^2t}{\sigma})) + e^{\mu t +\frac{\sigma^2t^2}{2}}(-1)(-\sigma)\frac{1}{\sqrt{2\pi}}e^{-\frac{1}{2}(\frac{\mu+\sigma^2t}{\sigma})^2} \\
    & = (\mu+\sigma^2t) e^{\mu t +\frac{\sigma^2t^2}{2}} (1-\phi(-\frac{\mu+\sigma^2t}{\sigma})) + \frac{\sigma}{\sqrt{2\pi}}e^{-\frac{\mu^2+2\mu\sigma^2t+\sigma^4t^2-2\sigma^2(\mu t + \sigma^2t^2/2)}{2\sigma^2}}\\
    & = (\mu+\sigma^2t) e^{\mu t+\frac{\sigma^2t^2}{2}} (1-\phi(-\frac{\mu+\sigma^2t}{\sigma})) + \frac{\sigma}{\sqrt{2\pi}}e^{-\frac{\mu^2}{2\sigma^2}}
\end{aligned}
\end{equation}

\begin{equation}
\begin{aligned}
\frac{d^2}{dt^2}M_X(t)
    & = (\sigma^2+(\mu+\sigma^2t)^2)e^{\mu t+\frac{\sigma^2 t^2}{2}} (1-\phi(-\frac{\mu+\sigma^2t}{\sigma})) \\
    & + (\mu+\sigma^2t) \ e^{\mu t+\frac{\sigma^2t^2}{2}} \frac{d}{dt} [1-\phi(-\frac{\mu+\sigma^2t}{\sigma})] \\
    & = (\sigma^2+(\mu+\sigma^2t)^2)e^{\mu t+\frac{\sigma^2 t^2}{2}} (1-\phi(-\frac{\mu+\sigma^2t}{\sigma}))  + (\mu+\sigma^2t) \frac{\sigma}{\sqrt{2\pi}}e^{-\frac{\mu^2}{2\sigma^2}}
\end{aligned}
\end{equation}

Now, we finally obtain:

\begin{equation}
\begin{aligned}
\mathbb{E}(X)  
    & = \frac{d}{dt}M_X(t)|_{t=0} = \mu(1-\phi (-\frac{\mu}{\sigma})) + \frac{\sigma}{\sqrt{2\pi}}e^{-\frac{\mu^2}{2\sigma^2}} \\
\end{aligned}
\end{equation}

\begin{equation}
\begin{aligned}
\mathrm{Var}(X) 
    & = \frac{d^2}{dt^2}M_X(t)|_{t=0} - (\frac{d}{dt}M_X(t)|_{t=0})^2 \\
    & = (\sigma^2+\mu^2)(1-\phi(-\frac{\mu}{\sigma})) + \mu \frac{\sigma}{\sqrt{2\pi}}e^{-\frac{\mu^2}{2\sigma^2}}  - (\mu(1-\phi (-\frac{\mu}{\sigma})) + \frac{\sigma}{\sqrt{2\pi}}e^{-\frac{\mu^2}{2\sigma^2}})^2
\end{aligned}
\end{equation}

which proves the proposition.

\subsection{Proof of Theorem \ref{thm:basic}}
\label{prf:basic}

Let us denote $\mathrm{ReLU}(x^n)=x^{n+}$, $W^nx^{n+} +b^n = \hat{x}^{n+1}$, and $BatchNorm(\hat{x}^{n+1}) = x^{n+1}$. We first calculate the statistics of $\hat{x}^{n+1}$, which is estimated by the batch normalization layer.

\begin{equation}
\begin{aligned}
\hat{\mu}_j^{n+1} 
    & = \mathbb{E}_x[\hat{x}^{n+1}_j] \\
    & = \mathbb{E}_x[\sum_i (x^{n+}_i W^n_{ij}) + b^{n}_j]\\
    & = \sum_i (\mathbb{E}_x[x^{n+}_i] W^n_{ij}) + b^{n}_j\\
    & = \sum_i (\mu^n(1-\phi (-\frac{\mu^n}{\sigma^n})) + \frac{\sigma^n}{\sqrt{2\pi}}e^{-\frac{(\mu^n)^2}{2(\sigma^n)^2}} ) W^n_{ij} + b^{n}_j
\end{aligned}
\end{equation}

Note that this equation describes the sampling process of the batch normalization layer, therefore the statistics of model parameters($W^n_{ij}$, $b^n_{j}$) are not considered. Other than this part, we omit the subscript if obvious. Similarly, 

\begin{equation}
\begin{aligned}
(\hat{\sigma}_j^{n+1})^2 
    & = Var_x(\hat{x}^{n+1}_j) \\
    & = Var_x(\sum_i (x^{n+}_i W^n_{ij}) + b^{n}_j)\\
    & = \sum_i Var_x(x^{n+}_i W^n_{ij}) \\
    & = \sum_i (W^n_{ij})^2 Var_x(x^{n+}_i) \\
    & = (  ((\sigma^n)^2+(\mu^n)^2)(1-\phi(-\frac{\mu^n}{\sigma^n})) + \mu^n \frac{\sigma^n}{\sqrt{2\pi}}e^{-\frac{(\mu^n)^2}{2(\sigma^n)^2}}  \\
    & \quad - (\mu^n(1-\phi (-\frac{\mu^n}{\sigma^n})) + \frac{\sigma^n}{\sqrt{2\pi}}e^{-\frac{(\mu^n)^2}{2(\sigma^n)^2}})^2 )\sum_i (W^n_{ij})^2 \\
\end{aligned}
\end{equation}

And $x^{n+1}_j$ is given by:

\begin{equation}
\begin{aligned}
x^{n+1}_j 
    & = BatchNorm(\hat{x}^{n+1}) \\ 
    & = \frac{\hat{x}^{n+1}-\hat{\mu}_j^{n+1}}{\hat{\sigma}_j^{n+1}}\gamma^{n+1}_j + \beta^{n+1}_j 
\end{aligned}
\end{equation}

where $\gamma^{n+1},\beta^{n+1}\in \mathbb{R}^{n_{d+1}}$ are affine transform parameters of batch normalization.

Now, we calculate the gradient of the input using the backpropagation algorithm. Since ReLU operation not only changes the activation but also blocks the gradient, they are treated separately. For indices such that $x^n_i \leq 0$, the gradient $g^n_i$ is just zero. Then, for each $i$ that is not blocked by ReLU (i.e. for $i \in \{ i: x^n_i > 0 \}$), we have:

\begin{equation}
\begin{aligned}
g_i^n 
    & = \frac{dL}{dx^n_i} = \sum_j \frac{dL}{dx^{n+1}_j}\frac{dx^{n+1}_j}{dx^n_i} \\
    & = \sum_j g^{n+1}_j \frac{dx^{n+1}_j}{d\hat{x}^{n+1}_j} \frac{d\hat{x}^{n+1}_j}{dx^n_i} \\
    & = \sum_j g^{n+1}_j \frac{\gamma^{n+1}_j}{\hat{\sigma}_j^{n+1}} W^n_{ij}
\end{aligned}
\end{equation}

Assuming the gradient $g^{n+1}_j$ is zero-centered and independent to other variables, we have

\begin{equation}
\begin{aligned}
\mathrm{Var}(g_i^n)
    & = \mathrm{Var}(\sum_j g^{n+1}_j \frac{\gamma^{n+1}_j}{\hat{\sigma}^{n+1}_j} W^n_{ij}) \\ 
    & = \sum_j \mathrm{Var}(g^{n+1}_j \frac{\gamma^{n+1}_j}{\hat{\sigma}^{n+1}_j} W^n_{ij}) \\ 
    & = \sum_j \mathbb{E}[(g^{n+1}_j \frac{\gamma^{n+1}_j}{\hat{\sigma}^{n+1}_j} W^n_{ij})^2] - \mathbb{E}[g^{n+1}_j \frac{\gamma^{n+1}_j}{\hat{\sigma}^{n+1}_j} W^n_{ij}]^2   \\ 
    & = \sum_j \mathrm{Var}(g^{n+1}_j) (\gamma^{n+1}_j)^2 \mathbb{E}[\frac{1}{(\hat{\sigma}^{n+1}_j)^2}] \mathbb{E}[(W^n_{ij})^2] \\ 
\end{aligned}
\end{equation}

since $f(x)=1/x$ is convex at $x>0$, we can apply Jensen's inequality \citep{jensen}. 

\begin{equation}
\begin{aligned}
    & \geq \sum_j \mathrm{Var}(g^{n+1}_j) (\gamma^{n+1}_j)^2 \frac{1}{\mathbb{E}[(\hat{\sigma}^{n+1}_j)^2]} \mathbb{E}[(W^n_{ij})^2] \\
    & =  \sum_j \mathrm{Var}(g^{n+1}_j) (\sigma^{n+1}_j)^2 \frac{\mathbb{E}[(W^n_{ij})^2]}{\mathbb{E}[\sum_k \mathrm{Var}(x_k^{n+}) (W^n_{kj})^2]} \\
    & = \sum_j (\sigma^{n+1}_j)^2 \mathrm{Var}(g^{n+1}_j) \frac{\sigma_w^2}{\mathbb{E}[\sum_k  \mathrm{Var}(x_k^{n+})] \sigma_w^2} \\
    & = \frac{1}{\sum_k \mathrm{Var}(x_k^{n+})} \sum_j (\sigma^{n+1}_j)^2 \mathrm{Var}(g^{n+1}_j)
\end{aligned}
\end{equation}

The equality holds when $\hat{\sigma}_j^{n+1}$ is identical for all $j$. Note that this formula is no longer dependent on the input index $i$, if $x_i^n$ is not blocked by ReLU. Since the probability of passing ReLU is $1-\phi(-\mu^n/\sigma^n)$ for any input index $i$, we have 

\begin{equation}
\begin{aligned}
\mathbb{E}&[\sum_i  (\sigma_i^n)^2  \mathrm{Var}(g^n_i)] / \sum_j (\sigma^{n+1}_j)^2 \mathrm{Var}(g_j^{n+1}) ] \\
    & \geq \mathbb{E}[ \sum_i (1-\phi(-\frac{\mu^n}{\sigma^n})) (\sigma^n)^2 \frac{1}{\sum_k \mathrm{Var}(x_k^{n+})}  \sum_j (\sigma^{n+1}_j)^2 \mathrm{Var}(g^{n+1}_j)/\sum_j (\sigma^{n+1}_j)^2 \mathrm{Var}(g_j^{n+1})] \\
    & = (1-\phi(-\frac{\mu^n}{\sigma^n})) (\sigma^n)^2 / \Big( ((\sigma^n)^2+(\mu^n)^2)(1-\phi(-\frac{\mu^n}{\sigma^n})) + \mu^n \frac{\sigma^n}{\sqrt{2\pi}}e^{-\frac{(\mu^n)^2}{2(\sigma^n)^2}}  \\
    & \quad - (\mu^n(1-\phi (-\frac{\mu^n}{\sigma^n})) + \frac{\sigma^n}{\sqrt{2\pi}}e^{-\frac{(\mu^n)^2}{2(\sigma^n)^2}})^2 \Big)\\
    & = \frac{1-\phi(-\frac{\mu^n}{\sigma^n})} { (1+(\frac{\mu^n}{\sigma^n})^2)(1-\phi(-\frac{\mu^n}{\sigma^n})) + \frac{\mu^n}{\sigma^n} \frac{1}{\sqrt{2\pi}}exp(-\frac{1}{2} (\frac{\mu^n}{\sigma^n})^2)  - (\frac{\mu^n}{\sigma^n}(1-\phi (-\frac{\mu^n}{\sigma^n})) +   \frac{1}{\sqrt{2\pi}}exp(-\frac{1}{2} (\frac{\mu^n}{\sigma^n})^2) )^2  } \\ 
    & = \frac{\erfc(-R)} { (1+2R^2)\erfc(-R) + 2R e^{-R^2}/\sqrt{\pi}  - (R\erfc(-R) +   e^{-R^2}/\sqrt{\pi} )^2  } 
\end{aligned}
\end{equation}

where $R := \frac{\mu^n}{\sqrt{2}\sigma^n}$ and $\erfc (x) = \frac{2}{\pi} \int_{x}^{\infty}\frac{1}{\sqrt{2\pi}}e^{-x^2/2} = 2-2\phi(x\sqrt{2})$ is a complementary error function. 

\subsection{Proof of Corollary \ref{thm:R}}
\label{prf:R}
In case of $R\to -\infty$, observing $\erfc (x)$ exponentially converges to zero at $x \to \infty$, both denominator and numerator converges to zero. And the derivative of the denominator

\begin{equation}
\begin{aligned}
\frac{d}{dR} &\Big( (1+2R^2)\erfc(-R) + 2R \frac{e^{-R^2}}{\sqrt{\pi}}  - (R\erfc(-R) +   \frac{e^{-R^2}}{\sqrt{\pi}} )^2            \Big) \\
    &= 4R \erfc (-R) + (1+2R^2)2 \frac{e^{-R^2}}{\sqrt{\pi}} + (2-4R^2) \frac{e^{-R^2}}{\sqrt{\pi}} \\
    & \quad -2(R \erfc (-R) +  \frac{e^{-R^2}}{\sqrt{\pi}})(\erfc (-R) + R \cdot 2 \frac{e^{-R^2}}{\sqrt{\pi}} -2R  \frac{e^{-R^2}}{\sqrt{\pi}}) \\
    & = 4R \erfc (-R) + \frac{e^{-R^2}}{\sqrt{\pi}}(4-2\erfc (-R)) -2R \erfc(-R)^2 \\
    & = 2(2-\erfc (-R)) (\erfc (-R) + \frac{e^{-R^2}}{\sqrt{\pi}})
\end{aligned}
\end{equation}

is strictly positive in $(-\infty, \infty)$. Therefore we can utilize L'Hôpital's rule:

\begin{equation}
\begin{aligned}
\lim_{R \to -\infty} & \frac{\erfc(-R)} { (1+2R^2)\erfc(-R) + 2R e^{-R^2}/\sqrt{\pi}  - (R\erfc(-R) +   e^{-R^2}/\sqrt{\pi} )^2} \\
    & = \lim_{R \to -\infty} \frac{\frac{e^{-R^2}}{\sqrt{\pi}}}{(2-\erfc (-R)) (\erfc (-R) + \frac{e^{-R^2}}{\sqrt{\pi}})} \\
    & = \lim_{R \to -\infty} \frac{1}{(2-\erfc (-R)) (\erfc (-R) / (\frac{e^{-R^2}}{\sqrt{\pi}})+1)} = \infty
\end{aligned}
\end{equation}

since $2-\erfc (-R) \to 0+$ and $\erfc (-R)/e^{-R^2} \to 0$. 

\begin{equation}
\begin{aligned}
\end{aligned}
\end{equation}

\subsection{Proof of Theorem \ref{thm:upper_bound}}
\label{prf:upper_bound}
We start from eq.(17) and get the upper bound of $\mathbb{E}[1/(\hat{\sigma}_j^{n+1})^2]$ instead of lower bound. Since $W_{ij}$ is normally distributed, $\sum W_{ij}$ can be viewed as a noncentral chi-squared distribution with nontrivial variance. Let $X$ be a chi-squared distribution with $d_n$ degrees of freedom. For any $0<c<1$ and $t<0$,

\begin{equation}
\begin{aligned}
P(\sum_i (\frac{W_{ij}^n}{\sigma_w})^2 < d_n \delta) &\leq P(X < d_n \delta)  = P(e^{tX} > d_n t\delta) \\ 
 & \leq M_X(t)e^{-t\delta n} = e^{-(n/2)ln(1-2t)-t\delta n}
\end{aligned}
\end{equation}

The first inequality comes from the fact that $W_{ij}/\sigma_w$ may not be zero-centered, and the second one is Markov's inequality. Setting $t=(1-1/\delta )/2$,

\begin{equation}
\begin{aligned}
P(\sum_i (W_{ij}^n)^2 < \sigma_w^2 d_n \delta) &\leq e^{-(d_n/2)(\delta-1-\mathrm{ln}(\delta))}
\end{aligned}
\end{equation}

Using the sharpened Jensen's inequality $\mathbb{E}[f(X)]-f(\mathbb{E}[X]) \leq \sigma^2_X \mathrm{sup}(f''(x)/2)$ proposed by \citep{sharpened_jensen}, at least with probability $1-e^{-(d_n/2)(\delta -1-\mathrm{ln}(\delta))}$, we have

\begin{equation}
\begin{aligned}
\mathbb{E}[\frac{1}{\hat{\sigma}_j^{n+1}}]\mathbb{E}[(W_{ij}^n)^2] & = \frac{1}{\sum_k \mathrm{Var}(x_k^{n+})} \mathbb{E}[\frac{1}{(W_{ij}^n)^2}] \mathbb{E}[(W_{ij}^n)^2] \\
    & \leq \frac{1}{\sum_k \mathrm{Var}(x_k^{n+})} ( \frac{\mathbb{E}[(W_{ij}^n)^2]}{\mathbb{E}[(W_{ij}^n)^2]} + \mathbb{E}[(W_{ij}^n)^2]2\sigma_w^4(1+2\frac{\mu_w^2}{\sigma_w^2} ) \frac{2}{2(\sigma_w^2 d_n \delta)^3}) \\
    &= \frac{1}{\sum_k \mathrm{Var}(x_k^{n+})} (1 + \sigma_w^2(1+\frac{\mu_w^2}{\sigma_w^2}) \sigma_w^4(1+2\frac{\mu_w^2}{\sigma_w^2}) \frac{2}{\sigma_w^6 d_n^3 \delta^3}) \\ 
    & = \frac{1}{\sum_k \mathrm{Var}(x_k^{n+})} (1 + \frac{2(1+\frac{\mu_w^2}{\sigma_w^2}) (1+2\frac{\mu_w^2}{\sigma_w^2})}{d_n^3 \delta^3})
\end{aligned}
\end{equation}

There is at least $(1-\mathrm{exp}(-(d_n/2)(\delta -1-\mathrm{ln}(\delta))))^{d_{n+1}} \geq 1- d_{n+1} \mathrm{exp}(-(d_n/2)(\delta -1-\mathrm{ln}(\delta)))$ chance that the inequality holds for all $1 \leq j \leq d_{n+1}$. The remaining part is the same as Theorem \ref{thm:basic}. 

\subsection{Proof of Theorem \ref{thm:corr}}
\label{prf:corr}
Since inputs($x^n$) are perfectly correlated, let $b \sim N(0,1)$ and $x^{n}=b t$, where $t \in \mathbb{R}^{n_{d}}$ is treated as a constant. Similar to Theorem 1, We denote $\mathrm{ReLU}(x^n)=x^{n+}$, $W^nx^{n+} +b^n = \hat{x}^{n+1}$, and $BatchNorm(\hat{x}^{n+1}) = x^{n+1}$ but ignore the bias term $b^n$ since it doesn't affect the variance and is immediately normalized by the batch normalization layer. We first calculate the mean and variance of $\hat{x}^{n+1}$. As the sign of $x^{n}$ is reversed corresponding to positive and negative $t$, they are treated separately.

\begin{equation}
\begin{aligned}
\hat{\mu}_j^{n+1} 
    & = \mathbb{E}_x[\hat{x}^{n+1}_j] \\
    & = \mathbb{E}_x[\hat{x}^{n+1}_j | b \geq 0]P(b \geq 0) + E_x[\hat{x}^{n+1}_j | b \leq 0]P(b \leq 0) \\
    & = \frac{1}{2}\mathbb{E}_x[\sum_i W_{ij}^n \mathrm{ReLU}(t_i)b | b \geq 0] + \frac{1}{2}E_x[\sum_i W_{ij}^n \mathrm{ReLU}(-t_i)(-b) | b \leq 0]\\
    & = \frac{1}{2}\frac{2}{\sqrt{2\pi}}\sum_i W_{ij}^n (\mathrm{ReLU}(t_i)+\mathrm{ReLU}(-t_i)) \\
    & = \frac{1}{\sqrt{2\pi}}\sum_i W_{ij}^n |t_i|\\
\end{aligned}
\end{equation}

The conditional expectation of $b$ is calculated using Proposition 1. Then, we directly calculate the variance before normalization.   

\begin{equation}
\begin{aligned}
(\hat{\sigma}_j^{n+1})^2 
    & =\int_{-\infty}^{\infty} (\hat{x}^{n+1}_j - \hat{\mu}_j^{n+1})^2 P(b)db  \\
    & = \int_{-\infty}^0 (\sum_i W_{ij}^n \mathrm{ReLU}(-t_i)(-b) - \hat{\mu}_j^{n+1})^2 \frac{1}{\sqrt{2\pi}}e^{-\frac{b^2}{2}}db \\
    & + \int_{0}^{\infty} (\sum_i W_{ij}^n \mathrm{ReLU}(t_i)(b) - \hat{\mu}_j^{n+1})^2 \frac{1}{\sqrt{2\pi}}e^{-\frac{b^2}{2}}db \\
    & = \int_{-\infty}^0 (-N_j b - \hat{\mu}_j^{n+1})^2 \frac{1}{\sqrt{2\pi}}e^{-\frac{b^2}{2}}db  + \int_{0}^{\infty} (P_j b - \hat{\mu}_j^{n+1})^2 \frac{1}{\sqrt{2\pi}}e^{-\frac{b^2}{2}}db
\end{aligned}
\end{equation}

where $N_j := \sum_i W_{ij}^n \mathrm{ReLU}(-t_i)$ and $P_j := \sum_i W_{ij}^n \mathrm{ReLU}(t_i)$. Using the formula $\int (ax-b)^2 e^{-x^2/2}dx = \sqrt{\frac{\pi}{2}}(a^2+b^2)erf(\frac{x}{\sqrt{2}})-ae^{-x^2/2}(ax-2b)+C$, where $erf(x) = 2\phi(x\sqrt{2})-1$ is an error function, we obtain:

\begin{equation}
\begin{aligned}
    & = \frac{1}{\sqrt{2\pi}} \Big[\sqrt{\frac{\pi}{2}}(N_j^2+(\hat{\mu}_j^{n+1})^2)erf(\frac{x}{\sqrt{2}})+Ne^{-\frac{x^2}{2}}(-N_jx-2\hat{\mu}_j^{n+1})  \Big]^0_{-\infty} \\ 
    & + \frac{1}{\sqrt{2\pi}} \Big[\sqrt{\frac{\pi}{2}}(P_j^2+(\hat{\mu}_j^{n+1})^2)erf(\frac{x}{\sqrt{2}})-Pe^{-\frac{x^2}{2}}(P_jx-2\hat{\mu}_j^{n+1}) \Big]^{\infty}_{0}\\
    & = \frac{1}{\sqrt{2\pi}}[-2N_j\hat{\mu}_j^{n+1}+\sqrt{\frac{\pi}{2}}(N_j^2+(\hat{\mu}_j^{n+1})^2) + \sqrt{\frac{\pi}{2}}(P_j^2+(\hat{\mu}_j^{n+1})^2) - 2P_j\hat{\mu}_j^{n+1}] \\
    & = \frac{1}{2}(N_j^2+P_j^2)+(\hat{\mu}_j^{n+1})^2- \sqrt{\frac{2}{\pi}}\hat{\mu}_j^{n+1}(P_j+N_j)
\end{aligned}
\end{equation}

 Since $\{W_{ij}\}$ are independent of each other and $\mathbb{E}[N_j] = \mathbb{E}[P_j] = \mathbb{E}[\hat{\mu}_j^{n+1}] = 0$, we obtain:

\begin{equation}
\begin{aligned}
\mathbb{E}[(\hat{\sigma}_j^{n+1})^2]
    & = \mathbb{E}[\frac{1}{2}(N_j^2+P_j^2)+(\mu_j^{n+1})^2- \sqrt{\frac{2}{\pi}}\mu_j^{n+1}(P_j+N_j)] \\ 
    & = \frac{1}{2}(\mathbb{E}[N_j^2]+\mathbb{E}[P_j^2])+\mathbb{E}[(\mu_j^{n+1})^2] -\sqrt{\frac{2}{\pi}}\mathbb{E}[\mu_j^{n+1}(P_j+N_j)] \\ 
    & = \frac{1}{2}(\mathrm{Var}(N_j)+\mathrm{Var}(P_j))+\mathrm{Var}(\mu_j^{n+1})-0 \\
    & = \frac{1}{2}(\sigma_w^2 \sum_i \mathrm{ReLU}(t_i)^2 +\sigma_w^2 \sum_i \mathrm{ReLU}(-t_i)^2) + \frac{1}{2\pi} \sigma_w^2 \sum_i |t_i|^2 \\ 
    & = \sigma_w^2 (\frac{1}{2}+\frac{1}{2\pi})\sum_i t_i^2 \\
\end{aligned}
\end{equation}

where $\sigma_w$ denotes the standard deviation of the weight. Let $I^+ := \{ i: x^n_i > 0 \}$. For $i \in I^+$, we have:

\begin{equation}
\begin{aligned}
\mathrm{Var}(g_i^n)
    & = \mathrm{Var}(\sum_j g^{n+1}_j \frac{\gamma^{n+1}_j}{\hat{\sigma}^{n+1}_j} W^n_{ij}) \\ 
    & = \sum_j \mathrm{Var}(g^{n+1}_j \frac{\gamma^{n+1}_j}{\hat{\sigma}^{n+1}_j} W^n_{ij}) \\ 
    & = \sum_j \mathrm{Var}(g^{n+1}_j) (\mathbb{E}[\frac{\gamma_j^2}{(\hat{\sigma}^{n+1}_j)^2} (W^n_{ij})^2]-\mathbb{E}[\frac{\gamma^{n+1}_j}{\sigma^{n+1}_j} W^n_{ij}]^2) \\
    & = \sum_j \mathrm{Var}(g^{n+1}_j) (\gamma^{n+1}_j)^2 \mathbb{E}[\frac{1}{(\hat{\sigma}^{n+1}_j)^2}] \mathbb{E}[(W^n_{ij})^2] \\ 
    & \geq \sum_j \mathrm{Var}(g^{n+1}_j) (\gamma^{n+1}_j)^2 \frac{1}{\mathbb{E}[(\hat{\sigma}^{n+1}_j)^2]} \mathbb{E}[(W^n_{ij})^2] \\
    & = \frac{2\pi}{\pi+1} \sum_j (\sigma^{n+1}_j)^2 \mathrm{Var}(g^{n+1}_j) \frac{\sigma_w^2}{\sum_k t_k^2 \sigma_w^2} \\
    & = \frac{2\pi}{\pi+1} \frac{1}{\sum_k t_k^2} \sum_j (\sigma^{n+1}_j)^2 \mathrm{Var}(g^{n+1}_j)
\end{aligned}
\end{equation}

The equality holds when $\hat{\sigma}_j^{n+1}$ are identical for all $j$. Finally, we obtain:

\begin{equation}
\begin{aligned}
\mathbb{E}&[\sum_i  (\sigma_i^n)^2  \mathrm{Var}(g^n_i) / \sum_j (\sigma^{n+1}_j)^2 \mathrm{Var}(g^{n+1}_j)] \\
    & \geq \mathbb{E}[\frac{1}{2}\sum_i t_i^2  \frac{2\pi}{\pi+2} \frac{1}{\sum_k t_k^2} \sum_j (\sigma^{n+1}_j)^2 \mathrm{Var}(g^{n+1}_j) / \sum_j (\sigma^{n+1}_j)^2 \mathrm{Var}(g^{n+1}_j)] \\ 
    & = \frac{1}{2}\frac{2\pi}{\pi+1} \frac{\sum_i t_i^2}{\sum_k t_k^2}\\ 
    & = \frac{\pi}{\pi+1}
\end{aligned}
\end{equation}

This proves the theorem.

\subsection{Proof of Lemma \ref{thm:hessian}}
\label{prf:hessian}

Let $x^N$ be an output of the network, which is scalar by definition. Since a ReLU network is piecewise linear, $dx^N/dx_i^n$ is not a function of $x_i^n$, for arbitrary $n \leq N$ and $i \leq d_n$. Let $l$ be a nonlinear loss function and $L = l(dx^N)$.

\begin{equation}
\begin{aligned}
h_i^n := \frac{d^2 L}{(dx_i^n)^2} 
    & = \frac{d^2 l(x^N)}{(dx_i^n)^2} = \frac{d}{dx_i^n} \Big( \frac{dx^N}{dx_i^n} l'(x^N) \Big) = \Big( \frac{dx^N}{dx_i^n} \Big)^2 l''(x^N) = \Big( \frac{g_i^n}{l'(x^N)} \Big)^2 l''(x^N)
\end{aligned}
\end{equation}

If $g$ is normally distributed, statistics of $h$ can be calculated with a chi-square distribution with $1$ degree of freedom.

\begin{equation}
\begin{aligned}
\frac{\mathrm{Var}(h_i^n)}{\mathrm{Var}(h_j^m)} &= \frac{\frac{l''(x^N)^2}{l'(x^N)^4} \mathrm{Var}((g_i^n)^2)}{\frac{l''(x^N)^2}{l'(x^N)^4} \mathrm{Var}((g_j^m)^2)} = \frac{\mathrm{Var}(g_i^n)^2 2 \cdot 1}{\mathrm{Var}(g_j^m)^2 2 \cdot 1} = \Big( \frac{\mathrm{Var}(g_i^n)}{\mathrm{Var}(g_j^m)} \Big)^2 
\end{aligned}
\end{equation}

\section{Experimental Details}
\label{appendix:detail}
In all experiments, we employed Stochastic Gradient Descent with a momentum of 0.9 and weight decay rate of $5 \times 10^{-4}$. We employed a learning rate of $0.1$ for a batch size of $128$, which was then linearly scaled with batch size. \citep{linear_scaling} For the CIFAR10 experiment, we utilized the ResNet50 architecture \citep{ResNet}. To accommodate the smaller image size, we removed the first pooling layer and stride operation of the first convolutional layer. The input images were normalized using the mean values of (0.4914, 0.4822, 0.4465) and standard deviation values of (0.2023, 0.1994, 0.2010). Random cropping with a padding size of 4 and random horizontal flips were applied as data augmentation techniques. In the case of using WarmUp \citep{warmup}, the learning rate was linearly increased over the first 10 epochs for the CIFAR10 experiment (20 or 30 epochs for a `longer' WarmUp) and two epochs for the ImageNet experiment. Subsequently, a cosine learning rate decay without restarts \citep{cosine} was employed.

In the ImageNet experiment, the input images were normalized using the mean values of (0.485, 0.456, 0.406) and standard deviation values of (0.229, 0.224, 0.225). As data augmentation, we applied RandomResizedCrop \citep{ResNet} with an image size of 224, followed by a random horizontal flip. During the evaluation step, the images were resized to a size of 256 and center-cropped to a size of 224.

All experiments were conducted using a single RTX2080 Ti 12GB GPU. For experiments with large batch sizes, we utilized a computational trick where the parameter update was performed every n steps to effectively increase the batch size by n times. This approach is almost identical to using a batch size that is n times larger, with the exception of slightly higher sampling errors from normalization layers.

LAMB \citep{LAMB}, a modified version of LARS based on the Adam optimizer \citep{adam}, was not considered in these experiments. The Adam optimizer requires an additional hyperparameter and exhibits different characteristics, making a fair comparison with an SGD optimizer more complex. While Adam is known for its fast convergence, a basic SGD optimizer with momentum often outperforms it when employing a longer training schedule. Although the entire experiment could be repeated using an Adam-family optimizer, we do not consider it necessary at this time.

Nesterov momentum \citep{nesterov} has been shown to slightly and consistently improve performance across various setups. However, we have not found evidence that it specifically overcomes training instability. \citep{extrapsgd} proposed a method to enhance Nesterov momentum, but it is primarily targeted for distributed learning and does not demonstrate improvement in the single-node case. Similarly, \citep{CLARS} suggests using per-sample gradients instead of minibatch gradients, primarily in the context of distributed learning. We observed some benefits in specific cases, but calculating per-sample gradients was not readily available in popular deep learning frameworks such as PyTorch \citep{pytorch} or TensorFlow \citep{tensorflow}. Therefore, we adopted a module from Opacus \citep{opacus}, originally developed for differential privacy, to acquire the per-sample gradients. This module is the best option currently known to us, despite the fact that it incurs at least a doubling of time and memory requirements.

\begin{figure}
\centering
\includegraphics[width=0.7\textwidth]{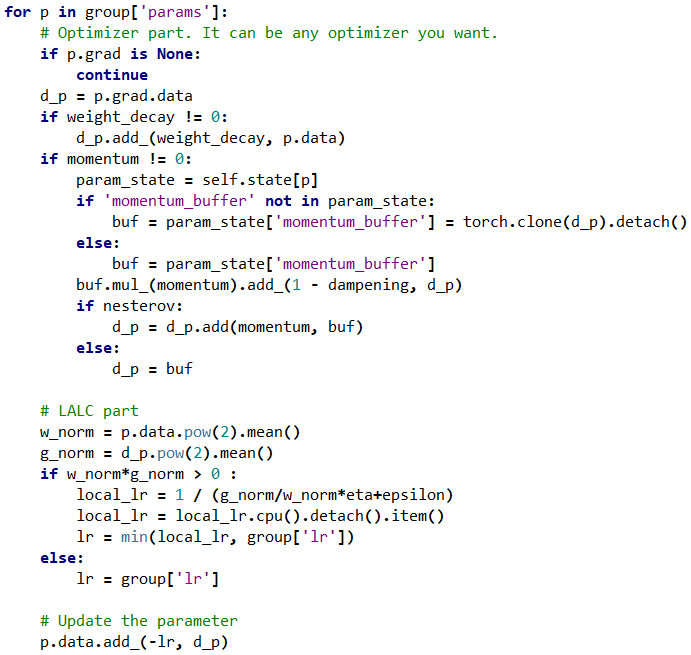}
\caption{ PyTorch \citep{pytorch} implementation of LALC. It can easily be implemented by adding a few lines of code to the existing optimizer.}
\end{figure}

\begin{table}[h]
\centering
\begin{tabular}{|c|c|c|c|c|c|}
\hline
Batch size & LARS  & CLARS  & LAMBC & AGC & LALC(ours)  \\ \hline
128        & $10^{-2}$  & $10^{-2}$   & $10^{-2}$ & $10^{-1}$ & $10^{3}$     \\ \hline
2048       & $10^{-3}$ & $10^{-3}$  & $10^{-2}$  & $10^{-1}$ & $10^{3}$    \\ \hline
4096       & $10^{-3}$ & $10^{-3}$  & $10^{-2}$  & $10^{-2}$ & $10^{3}$    \\ \hline
8192       & $10^{-3}$ & $10^{-4}$ & $10^{-2}$  & $10^{-2}$ & $2 \cdot 10^{3}$    \\ \hline
\end{tabular}
\caption{Coefficient $\eta$ used in our experiment. We used $\epsilon=10^{-3}$ for AGC, $10^0$ for LALC, and small numbers like $10^{-8}$ for others.}
\end{table}

\begin{table}[h]
\centering
\begin{tabular}{|l|llll|}
\hline
                & \multicolumn{4}{l|}{Batch size}                                                              \\ \cline{2-5} 
Method          & \multicolumn{1}{l|}{128}   & \multicolumn{1}{l|}{2048}  & \multicolumn{1}{l|}{4096}  & 8192  \\ \hline
SGD             & \multicolumn{1}{l|}{94.64 $\pm$ 0.31} & \multicolumn{1}{l|}{93.44 $\pm$ 0.31} & \multicolumn{1}{l|}{91.75 $\pm$ 0.59} & 71.17 $\pm$ 11.32 \\ \hline
SGD+WarmUp      & \multicolumn{1}{l|}{\textbf{95.48 $\pm$ 0.24}} & \multicolumn{1}{l|}{94.02 $\pm$ 0.27}  & \multicolumn{1}{l|}{92.13 $\pm$ 0.75} & 72.52 $\pm$ 4.68 \\ \hline
Nesterov+WarmUp & \multicolumn{1}{l|}{95.43 $\pm$ 0.12} & \multicolumn{1}{l|}{93.45 $\pm$ 0.18} & \multicolumn{1}{l|}{90.09 $\pm$ 0.89} & 62.62 $\pm$ 2.34 \\ \hline
LARS            & \multicolumn{1}{l|}{94.41 $\pm$ 0.18} & \multicolumn{1}{l|}{93.16 $\pm$ 0.16} & \multicolumn{1}{l|}{93.98 $\pm$ 0.21} & 88.22 $\pm$ 1.77  \\ \hline
LARS+WarmUp     & \multicolumn{1}{l|}{94.81 $\pm$ 0.09} & \multicolumn{1}{l|}{93.17 $\pm$ 0.04} & \multicolumn{1}{l|}{93.54 $\pm$ 0.09} & 92.98 $\pm$ 0.19  \\ \hline
LARS+WarmUp (20 epochs)     & \multicolumn{1}{l|}{94.74 $\pm$ 0.19} & \multicolumn{1}{l|}{93.25 $\pm$ 0.31} & \multicolumn{1}{l|}{93.51 $\pm$ 0.05} & 93.53 $\pm$ 0.24 \\ \hline
LARS+WarmUp (30 epochs)     & \multicolumn{1}{l|}{95.01 $\pm$ 0.09} & \multicolumn{1}{l|}{93.12 $\pm$ 0.19} & \multicolumn{1}{l|}{93.55 $\pm$ 0.10} & 93.70 $\pm$ 0.45 \\ \hline
CLARS           & \multicolumn{1}{l|}{92.59 $\pm$ 0.17} & \multicolumn{1}{l|}{93.51 $\pm$ 0.18} & \multicolumn{1}{l|}{93.61 $\pm$ 0.33} & 87.74 $\pm$ 0.79 \\ \hline
AGC           & \multicolumn{1}{l|}{95.36 $\pm$ 0.15} & \multicolumn{1}{l|}{94.36 $\pm$ 0.17} & \multicolumn{1}{l|}{93.49 $\pm$ 0.14} & 91.86 $\pm$ 0.33 \\ \hline
LAMBC      & \multicolumn{1}{l|}{95.18 $\pm$ 0.02}  & \multicolumn{1}{l|}{94.39 $\pm$ 0.11} & \multicolumn{1}{l|}{94.19 $\pm$ 0.06} & 93.73 $\pm$ 0.14 \\ \hline
LAMBC+WarmUp     & \multicolumn{1}{l|}{95.31 $\pm$ 0.14} & \multicolumn{1}{l|}{94.68 $\pm$ 0.17} & \multicolumn{1}{l|}{94.15 $\pm$ 0.12} & 93.49 $\pm$ 0.18 \\ \hline
LAMBC+Nesterov   & \multicolumn{1}{l|}{95.21 $\pm$ 0.13} & \multicolumn{1}{l|}{94.42 $\pm$ 0.14} & \multicolumn{1}{l|}{94.12 $\pm$ 0.21} & 93.28 $\pm$ 0.25 \\ \hline
LALC(ours)   & \multicolumn{1}{l|}{95.13 $\pm$ 0.03} & \multicolumn{1}{l|}{\textbf{95.07 $\pm$ 0.15}} & \multicolumn{1}{l|}{\textbf{94.82 $\pm$ 0.16}} & \textbf{94.15 $\pm$ 0.14} \\ \hline
\end{tabular}
\caption{The mean $\pm$ standard deviation of validation accuracy in the CIFAR10 experiment (Figure \ref{fig:CIFAR}). All experiments are averaged over 3 runs.}
\end{table}

\end{document}